\documentclass{article}

\usepackage{arxiv}

\usepackage{times}

\usepackage{soul}
\usepackage{url}
\usepackage[hidelinks]{hyperref}
\usepackage[utf8]{inputenc}
\usepackage[small]{caption}
\usepackage{graphicx}
\usepackage{amsmath,amssymb}
\usepackage{booktabs}
\usepackage{tabularx}
\usepackage{listings}
\usepackage{tcolorbox}
\usepackage[ruled,vlined]{algorithm2e}
\usepackage{nicefrac}       
\usepackage{microtype}      
\usepackage{amsthm}
\usepackage{amsfonts}
\usepackage{pythonhighlight}

\usepackage[noabbrev,nameinlink]{cleveref}
\crefname{property}{property}{Property}
\creflabelformat{property}{(#1)#2#3}
\crefname{equation}{eq}{Eq}
\creflabelformat{equation}{(#1)#2#3}

\newtheorem{theorem}{Theorem}[section]

\newtheorem{lemma}[theorem]{Lemma}
\newtheorem{remark}[theorem]{Remark}
\newtheorem{prop}{Proposition}
\theoremstyle{definition}
\newtheorem{definition}{Definition}[section]
\newtheorem{assumption}{Assumption}

\title{On the Global Self-attention Mechanism for Graph Convolutional Networks}

\author{
  Chen Wang \\
  Department of Computer Science\\
  Rutgers University - New Brunswick\\
  Piscataway, NJ 08854, USA \\
  \texttt{chen.wang.cs@rutgers.edu} \\
  \And
 Chengyuan Deng \\
  Department of Computer Science\\
  Rutgers University - New Brunswick\\
  Piscataway, NJ 08854, USA \\
  \texttt{charles.deng@rutgers.edu} \\
}

\begin{document}
\maketitle

\begin{abstract}
Applying Global Self-attention (GSA) mechanism over features has achieved remarkable success on Convolutional Neural Networks (CNNs). However, it is not clear if Graph Convolutional Networks (GCNs) can similarly benefit from such a technique. In this paper, inspired by the similarity between CNNs and GCNs, we study the impact of the Global Self-attention mechanism on GCNs. We find that consistent with the intuition, the GSA mechanism allows GCNs to capture feature-based vertex relations regardless of edge connections; As a result, the GSA mechanism can introduce extra expressive power to the GCNs. Furthermore, we analyze the impacts of the GSA mechanism on the issues of overfitting and over-smoothing. We prove that the GSA mechanism can alleviate both the overfitting and the over-smoothing issues based on some recent technical developments. Experiments on multiple benchmark datasets illustrate both superior expressive power and less significant overfitting and over-smoothing problems for the GSA-augmented GCNs, which corroborate the intuitions and the theoretical results.
\end{abstract}


%

\section{Introduction}
\label{sec:Intro}
The emerge of Graph Convolutional Network (GCN) framework \cite{kipf2016semi} has prompted graph networks to be one of the most promising techniques in pursuing artificial general intelligence \cite{battaglia2018relational}. Inspired by the closely related field of Concolutional Neural Networks (CNNs), different attention and self-attention mechanisms have been proposed to improve the quality of information aggregation under the GCN framework (e.g. \cite{velivckovic2017graph}). Existing self-attention mechanisms in GCNs usually consider the feature information between neighboring vertices, and assign connection weights to each vertex accordingly \cite{velivckovic2017graph,lee2018graph}. This type of attention considers the local geometry as the edge connections of the graph, and exclude possible scenarios when a vertex could have strong correlations and influences with another without edge connection. To date, as we can see from a comprehensive survey \cite{wu2019comprehensive}, there has not yet been any significant work applying the Global Self-attention (GSA) mechanism to GCNs. \par

We notice that despite the absence of the theoretical study on the GSA mechanism on the GCNs, such mechanism has achieved remarkable success on the similar Convolutional Neural Networks (notably in \cite{zhang2018self}). Therefore, in this paper, inspired by the above observations, we study the effect of GSA mechanism on the GCN domain. The similarity between CNNs and GCNs (\cite{wu2019comprehensive}) makes the implementation of the GSA mechanism on GCNs straightforward. For CNNs operating on image, the GSA mechanism functions in the way to compute the inner products between pixel features. Likewise, for GCNs operating on graphs, we can view each vertex roughly as a pixel, and the local edge-based geometry as an analogy of convolutional kernels; Therefore, the GSA mechanism on GCNs can be achieved by the direct product between every pairs of nodes, regardless of edge connection. An intuitive illustration of the above process can be found in \Cref{fig:local_global}.

\begin{figure}[!h]
\centering
\includegraphics[width=0.75\textwidth]{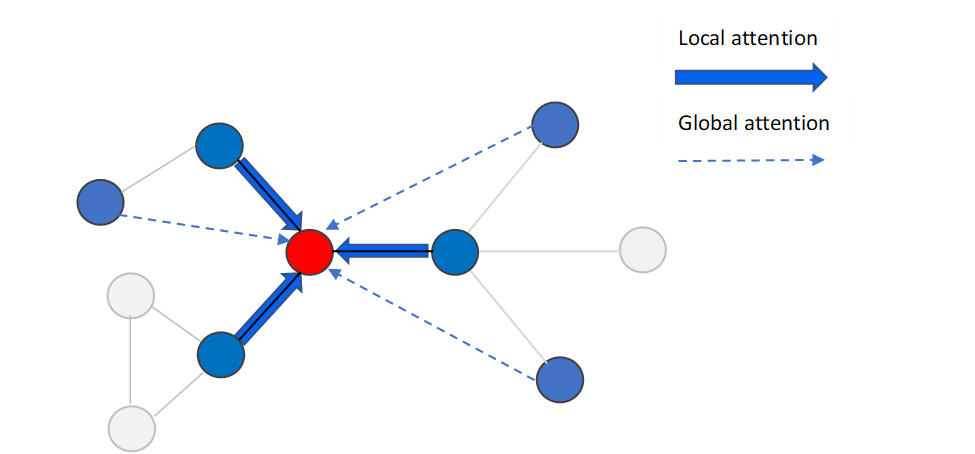}
\caption{\label{fig:local_global} Local \& Global Attention Mechanism in Graph Neural Networks. Local attention is only applied to neighbor nodes of the target node, while global attention includes every node in the graph}
\end{figure}
\par

Given the celebrated results for the GSA mechanism on CNNs \cite{zhang2018self}, it is reasonable to believe the GSA mechanism can give a performance boost for the GCN model. Similar to the intuition behind the GSA-augmented CNN, the GSA mechanism will give GCN the ability to capture long-range vertex dependencies and feature convolutions. This capacity can give the GCNs additional expressive power, especially when there exist pairs of vertices which are not connected by edges yet share similar features and are of the same class/pattern. We examine this intuition through the experiments on multiple benchmark datasets for node and graph classification tasks, and the empirical results affirm that by simply applying the GSA mechanism to plain GCN, it can outperform multiple carefully-designed advanced methods.  \par

In addition to the above intuition, we study the impacts for the GSA mechanism on overfitting and over-smoothing, two issues closely connected to the graph edge structure. Overfitting means the scenario when the testing accuracy decreases while the training accuracy still appears to climb up. Over-smoothing indicates the situation of exponentially-growing training losses as more layers are stacked. We prove that the GSA mechanism can mitigate both issues: On the overfitting problem, we decompose the loss of the GCN with the GSA mechanism into two regularization terms, and we prove that it is possible for the GSA mechanism to simulate the `edge dropout' process (\cite{rong2020dropedge}) to mitigate overfitting; On the front of over-smoothing, we prove that the GSA mechanism is roughly equivalent to interpolating a positive definite matrix to the original feature, and it can therefore reduce over-smoothing according to a recently-established theoretical framework from \cite{oono2019graph}. These theoretical results are verified by the experiments. \par

The rest of the paper is arranged as follows. \Cref{sec:review} provides a brief review of related work and the context of this paper; \Cref{sec:Method} outlines the details of applying the GSA mechanism to GCNs; The analysis on overfitting and over-smoothing is presented in \Cref{sec:analysis}, and experimental results are show in \Cref{sec:Experiments}; And finally, \Cref{sec:Conclusion} presents a general conclusion of our work.

\section{Related Work}
\label{sec:review}
We focus on the Graph Convolutional Networks based on Message-passing, which was initiated by the idea of approximating graph spectral convolution \cite{bruna2013spectral,defferrard2016convolutional,henaff2015deep} and popularized by the success of the standard GCN model \cite{kipf2016semi}. To date, there have been fruitful outcomes on novel graph networks stemmed from this strategy (i.e. \cite{hamilton2017inductive,li2018adaptive,levie2018cayleynets}). The mechanism of aggregating information from neighboring nodes can be deemed as a spatial-based approach \cite{gilmer2017neural} which is similar to the procedures in Convolutional Neural Networks (CNNs) \cite{wu2019comprehensive}. Consequently, the attention techniques in CNNs \cite{vaswani2017attention} have been widely examined from the perspective of GCNs (\cite{velivckovic2017graph,zhang2020context,Zhang2018GaANGA,wang2019sag,Xu2020SpatialTemporalTN}). \par
Despite the similarity between the two models, the attention methods in GCNs are mostly applied to local geometry with neighborhood connections, while CNN-based attention techniques are often applied to the global features. For instance, \cite{zhang2018self} provides a self-attention CNN as the generator of Generative Adversarial Networks (GANs) and shows celebrated capabilities in capturing long-range feature relations. The remarkable success indicates potential advance by applying the global attention mechanism to GCNs to get similar results. Moreover, \cite{Pei2020Geom} outlines one significant weakness of Message passing-based Graph Networks is the lack of `the ability to capture long-range dependencies'. Consequently, it can be reasonably conjectured that introducing global attention mechanism to GCNs will provide positive outcomes. Furthermore, some recent advances in GCNs have adopted ideas similar to the Global Self-attention mechanism (i.e., \cite{zhang2020context}); Nevertheless, the context and motivation for these studies are entirely different from ours.\par
From a theoretical perspective, the association between the GSA mechanism and the over-smoothing issue is significant. It has been long noticed that GCNs cannot be stacked as deep as CNNs without invoking negative effects (\cite{rong2019truly,Li2019DeepGCNsCG}). \cite{li2018deeper} provides an insight of this by showing the linear GCN is a special form of Laplacian smoothing and will converge to an feature-invariant point as the network goes deeper. Subsequently, \cite{oono2019graph} proves GCNs with Relu activation will converge to a feature-invariant space with a rate exponential to the maximum singular value of the convolutional filter. Our theorem for the GSA mechanism on over-smoothing is inspired by the above-mentioned work. The link between the GSA mechanism and overfitting is less obvious, but we notice that the GSA mechanism can simulate different methods (\cite{hamilton2017inductive,rong2020dropedge}) that are proved to alleviate overfitting. Therefore, we established our theoretical result on the GSA and overfitting based on this idea.

\section{Implementing Global Self-attention Mechanism on GCNs}
\label{sec:Method}
\subsection{Graph Convolution Networks}
Given a graph $\boldsymbol{G} = (\boldsymbol{V},\boldsymbol{E})$, a Graph Convolutional Networks takes as input a set of features $\boldsymbol{X}_i$ for every node (vertex) $v_{i}\in \boldsymbol{V}$, resulting in a node-feature matrix representation $\boldsymbol{X}\in \mathbb{R}^{n\times d}$ of a graph where $n$ is number of nodes/vertices and $d$ is number of features. We use the graph adjacency matrix $\boldsymbol{A}$ model with the size of $n\times n$ to represent the edge connections. The network layers are connected using a convolutional projection of input graph $\boldsymbol{X}$ with adjacency matrix $\boldsymbol{A}$. In practice, we add self-loop on each node to get $\boldsymbol{\hat{A}}=\boldsymbol{A}+\boldsymbol{I}$ with $\boldsymbol{I}$ as the identity matrix. Therefore, the feed-forward layer of a GCN can be expressed as
\begin{equation}
\label{equ:graphaggre}
\begin{aligned}
    \boldsymbol{H}^{(l+1)} &= f_{\boldsymbol{W}}(\boldsymbol{H}^{(l)},\boldsymbol{A}) \\
            &= \sigma(\boldsymbol{D}^{-\frac{1}{2}} \boldsymbol{\hat{A}} \boldsymbol{D}^{-\frac{1}{2}}\boldsymbol{H}^{(l)} \boldsymbol{W}^{(l)})\\
            &= \sigma( \boldsymbol{\tilde{A}}\boldsymbol{H}^{(l)} \boldsymbol{W}^{(l)}),
\end{aligned}
\end{equation}
where $\boldsymbol{W}^{(l)}$ is the weight matrix for convolution of the $l$-th layer and $\sigma(\cdot)$ is the activation function. Matrix $\boldsymbol{D}= \texttt{diag} \left( \sum_{j=1}^{N} \boldsymbol{\hat{A}}_{\cdot,j} \right)$ is the normalization matrix of $\hat{A}$, where $\texttt{diag}(\cdot)$ means the operation of expanding an $n$-length vector to a $n\times n$ matrix with the main diagonal filled by the elements of the vector.

\subsection{Global Self-attention Mechanism}
The Global Self-attention mechanism on images with Convolutional Neural Networks (CNNs) is discussed thoroughly in \cite{zhang2018self}. Transferring to the GCNs, on each layer, the self-attention layer takes the output of the previous layer $\boldsymbol{H}^{(l)}$ and calculate the influence of node $j$ on node $i$ with
\begin{equation}
\label{equ:attentionmask}
\begin{split}
    \beta_{i,j} = \texttt{softmax}_{j\in\{1,2,\cdots,n\}}[s_{i,j}] \\
    \text{where } s_{i,j}=(\boldsymbol{\hat{H}}^{(l)}_{i}\boldsymbol{W_{l}})(\boldsymbol{\hat{H}}^{(l)}_{j}\boldsymbol{W_{r}})^{T},
\end{split}
\end{equation}
where the calculation of $s_{i,j}$ is essentially a pair-wise production by summing over all the channels/features of the node. Matrices $\boldsymbol{W_{l}}$ and $\boldsymbol{W_{r}}$ serve the purpose of dimension-reduction to reduce the computational load and to provide additional flexibility to the trainable variables. We denote the result of the attention importance map as matrix $\boldsymbol{B}$, which is a $n \times n$ matrix. With the attention importance mask, the attention feature can be calculated with
\begin{equation}
\label{equ:attentionfeature}
    \boldsymbol{o}^{(l)}_{i} = (\sum_{j=1}^{N}\beta_{i,j}\boldsymbol{H}^{(l)}_{j}\boldsymbol{W}_{h})\boldsymbol{W}_{g},
\end{equation}
where $\boldsymbol{W}_{h}$ is the matrix to transform the input $\boldsymbol{H}^{(l)}$ to a lower dimension and $\boldsymbol{W}_{h}$ is the matrix to project the feature size back to the original. The operations of the above equation can be efficiently paralleled by re-writing it in the matrix production formula.  
\subsection{GSA Mechanism Augmented GCNs}
The GSA mechanism captures the feature information on a global level. Nevertheless, for graph inputs and networks, local geometry denoted by edge connections is also crucial (and it is the very characteristic that makes a \emph{graph} network). Thus, we perform an interpolation similar to \cite{zhang2018self}. The resulting Global Self-attention GCN layer goes as
\begin{equation}
\label{equ:layergraphconvolution}
    \boldsymbol{H}^{(l+1)} = \sigma((\boldsymbol{\tilde{A}}\boldsymbol{H}^{(l)} + \gamma\boldsymbol{O}^{(l)})\boldsymbol{W}^{(l)}),
\end{equation}
where $\boldsymbol{O}^{(l)}$ is the output of the global self-attention and $\gamma$ is a non-negative trainable parameter ($\gamma>0$) with initial value $0$ and $\boldsymbol{W}^{(l)}$ is the convolution/filter matrix of layer $l$. Notice that the attention feature will \emph{not} be processed by graph adjacency-based aggregation in the above operation. The necessity of $\boldsymbol{W}^{(l)}$ is questioned in \cite{wu2019simplifying} and one can drop this matrix if the computational resource is limited. Nevertheless, we keep the matrix here to serve a general purpose of applications and analysis. 
\section{Theoretical Analysis for the GSA mechanism on GCNs}
\label{sec:analysis}
In this section, we discuss the impact of GSA mechanism on the issues of overfitting and over-smoothing in Graph Convolutional Networks (GCNs). We argue that the GSA mechanism can alleviate both overfitting and over-smoothing. Specifically, we provide the intuition for the GSA mechanism to prevent GCNs from overfitting the local edge geometry in \Cref{subsec:overfit}, and show that the GSA mechanism could affect the GCN in the same way of DropEdge \cite{rong2020dropedge} under certain assumptions and model variations. Furthermore, we prove that with mild approximations and assumptions, the GSA mechanism is guaranteed to mitigate over-smoothing in \Cref{subsec:oversmooth}. \par
In the interest of the analysis, we re-formulate a simplified version of the layer in \Cref{equ:layergraphconvolution}. Notice that the linear attention transformation matrices $\boldsymbol{W}_{l}, \boldsymbol{W}_{r}, \boldsymbol{W}_{h}, \boldsymbol{W}_{g}$ are employed mainly for the purpose of reducing computational complexity and accelerating training. Thus, we simplify them to the identity matrix $\boldsymbol{I}$ in this section.

\subsection{Overfitting}
\label{subsec:overfit}
We first discuss the impact for GSA on remedy the overfitting problem. We consider the overfitting noise from the following source: In Graph Convolutional Network, we assume the vertices connected by an edge share the similar information, and the neighbors of a vertex should contain homogeneous features. However, in reality, this is not necessarily true. Therefore, if a vertex learn from an neighboring vertex with irrelevant or mixed information, it causes an overfitting problem to the local edge geometry. \par
Intuitively, the Global Self-attention mechanism can alleviate this issue in two ways. Firstly, from a geometry perspective, the GSA mechanism can aggregate information from `faraway' vertices regardless of the edge connections, and its effect can be viewed as a regularization to the graph local geometries; And secondly, from a feature perspective, the GSA mechanism can force the vertices with similar features to share information, which mitigates the noise from neighboring vertices with feature of different patterns. \par
The above intuition can be mathematically grounded as follows. Denoting the output of the last hidden layer as $\boldsymbol{H}^{(L)}=\boldsymbol{H}^{(L-1)}\boldsymbol{W}$ (without considering the activation function), and $\boldsymbol{H}^{(L-1)}$ is the hidden features before the last layer. The loss function can be denoted as:
\begin{equation}
\label{equ:loss-overall}
\mathcal{L}(\boldsymbol{\tilde{A}}\boldsymbol{H}^{(L)}+\gamma\boldsymbol{B}\boldsymbol{H}^{(L)}, \boldsymbol{y}),
\end{equation}
where $\boldsymbol{y}$ is the target (the labels of nodes or graph). Let $\boldsymbol{\bar{A}}$ be the complement of the unnormalized $\boldsymbol{A}$, which means $\boldsymbol{\bar{A}}_{i,j}=1 \leftrightarrow \boldsymbol{A}_{i,j}=0$ (i.e. the adjacency matrix of `no connections'). Also, denote $\boldsymbol{J}$ as the `all ones' matrix and $\boldsymbol{J} = \boldsymbol{I} + \boldsymbol{L}$, where $\boldsymbol{L}$ is the `all-ones except the main diagonal' matrix. The first argument (denoting as $\boldsymbol{\hat{y}}$) of the loss function of \Cref{equ:loss-overall} can be decomposed as the function of $\boldsymbol{\bar{A}}$:
\begin{align*}
\boldsymbol{\hat{y}} &= \boldsymbol{\tilde{A}}\boldsymbol{H}^{(L)} + \gamma \boldsymbol{\bar{A}}\circ\boldsymbol{B}\boldsymbol{H}^{(L)} + \gamma(\boldsymbol{J}-\boldsymbol{\bar{A}})\circ \boldsymbol{B}\boldsymbol{H}^{(L)}\\
&= \boldsymbol{\tilde{A}}\boldsymbol{H}^{(L)} + \gamma \boldsymbol{\bar{A}}\circ\boldsymbol{B}\boldsymbol{H}^{(L)} + \gamma(\boldsymbol{I}+\boldsymbol{L}-\boldsymbol{\bar{A}})\circ \boldsymbol{B}\boldsymbol{H}^{(L)}\\
&= \boldsymbol{\tilde{A}}\boldsymbol{H}^{(L)} +\gamma(\boldsymbol{\bar{A}}+\boldsymbol{L})\circ\boldsymbol{B}\boldsymbol{H}^{(L)} + \gamma (\boldsymbol{I}-\boldsymbol{\bar{A}})\circ \boldsymbol{B}\boldsymbol{H}^{(L)}\\
&=\big(\boldsymbol{\tilde{A}}+\gamma(\boldsymbol{\bar{A}}+\boldsymbol{L})\circ\boldsymbol{B}\big) \boldsymbol{H}^{(L)}  +  \gamma (\boldsymbol{I}-\boldsymbol{\bar{A}})\circ \boldsymbol{B}\boldsymbol{H}^{(L)},
\end{align*}
where $\circ$ represents the element-wise multiplication. In the above equation, let $\boldsymbol{A}'=\boldsymbol{\tilde{A}}+\gamma(\boldsymbol{\bar{A}}+\boldsymbol{L})\circ\boldsymbol{B}$, and this represents the \emph{geometric penalization} of the edge connections to avoid overfitting. Furthermore, for any super-additive loss function with property $\mathcal{L}(x+y, z)\gtrsim \mathcal{L}(x, z)+\mathcal{L}(y, 0)$, the loss function in \Cref{equ:loss-overall} can be decomposed into
\begin{equation*}
\mathcal{L}\big((\boldsymbol{\tilde{A}}+\gamma(\boldsymbol{\bar{A}}+\boldsymbol{L})\circ\boldsymbol{B}) \boldsymbol{H}^{(L)}, \boldsymbol{y}\big)  +  \mathcal{L}\big(\gamma (\boldsymbol{I}-\boldsymbol{\bar{A}})\circ \boldsymbol{B}\boldsymbol{H}^{(L)}, \boldsymbol{0} \big).
\end{equation*}
For the linear solution of the second term (i.e. $\gamma (\boldsymbol{I}-\boldsymbol{\bar{A}})\circ \boldsymbol{B}\boldsymbol{H}^{(L)}=\boldsymbol{0}$), it is equivalent to the minimization of the quadratic form:
\begin{equation*}
\frac{\gamma}{2}\cdot (\boldsymbol{H}^{(L)})^{T} (\boldsymbol{I}-\boldsymbol{\bar{A}})\circ\boldsymbol{B} (\boldsymbol{H}^{(L)}).
\end{equation*}
This term is proportional to
\begin{equation*}
\mathcal{L}_{\text{reg}=}\frac{\gamma}{2}\cdot\sum_{i,j}\mathbb{I}\big(\boldsymbol{A}_{i,j}=0\big)\cdot\langle \boldsymbol{h}^{(L-1)}_{i},\boldsymbol{h}^{(L-1)}_{i}\rangle\cdot\|\boldsymbol{h}^{(L)}_{i}-\boldsymbol{h}^{(L)}_{j}\|_{2}^{2}, 
\end{equation*}
which means the GSA mechanism will penalize the situation where disconnected vertices $v_{i}$ and $v_{j}$ share similar features in the $(L-1)$ layer but are processed to considerably different features. Putting the above together, by minimizing the target in \Cref{equ:loss-overall} with vanilla gradient descent, the GSA mechanism will encourage the upper bound of the following target:
\begin{equation}
\label{equ:regular-target}
\mathcal{L}\big((\underbrace{\boldsymbol{\tilde{A}}+\gamma(\boldsymbol{\bar{A}}+\boldsymbol{L})\circ\boldsymbol{B}}_{\text{geometry regularization}}) \boldsymbol{H}^{(L)}, \boldsymbol{y}\big) \qquad +  \underbrace{\mathcal{L}_{\text{reg}}}_{\text{feature regularization}},
\end{equation}
which supports the intuition for the GSA mechanism to alleviate overfitting.

Beyond the above insights, one can also understand the impact of the Global Self-attention mechanism on GCNs from the perspective of vertex subsampling \cite{hamilton2017inductive} and DropEdge \cite{rong2020dropedge}. These methods adopt techniques similar to Dropout, and they have been shown to alleviate overfitting. The DropEdge method follows a simple strategy to delete edges during convolution; and to see how the GSA mechanism can simulate this method, we provide an analysis that, under certain assumptions and simplifications, the impact for the GSA mechanism on a significant portion of the vertices is essentially to cancel the impact of another vertex (thus `drop an edge'). \par
The analysis begins with the definitions of the \emph{GSA feature relations} and \emph{feature influence}:
\begin{definition}
We define the \emph{feature influence} from $v_{i}$ to $v_{j}$ as the value (or vector) that $v_{i}$ can change on the feature of $v_{j}$ without normalization. Furthermore, we define the \emph{GSA feature relations} from $v_{i}$ to $v_{j}$ as a binary scalar from $\{0,1\}$ that denotes whether $v_{i}$ can change the features of $v_{j}$ under the Graph Self-attention mechanism. 
\end{definition}
\noindent
Note that \emph{feature influence} holds unconditionally for both GSA mechanism and graph convolution, while \emph{GSA feature relations} are dependent on the GSA mechanism. The exact process of the GSA mechanism is difficult to analyze since every vertex can affect each other. Therefore, we make some assumptions and simplifications:
\begin{assumption}
\label{assum:graph-feature}
We study $d$-regular graphs with following properties:
\begin{enumerate}
\item For any vertex $v$, only $r$ ($r< n$) other vertices will have non-zero \emph{GSA feature relations} towards $v$ (only $r$ other vertices can influence $v$ under the GSA mechanism).
\item The \emph{feature influence} between two vertices are either $+1$ (\emph{positive}) or $-1$ (\emph{negative}). That is to say, the influence from each vertex to another is either $\frac{1}{d}$ or $-\frac{1}{d}$.
\item For any vertex $v$, the feature influence from its neighbors $\mathcal{N}(v)$ are not homogeneous (there exist both \emph{positive} and \emph{negative}). 
\end{enumerate}
\end{assumption}
\noindent
\begin{remark}
We can see that the above simplifications and assumptions are realistic. The first assumption seems strong; Nevertheless, since the \texttt{softmax()} in \Cref{equ:attentionmask} will suppress the score of a majority of vertices, the number of vertices with non-zero relations to $v \in \boldsymbol{V}$ under the GSA mechanism should indeed be a small fraction. The second assumption restricts the feature influence to $+1$ and $-1$; And usually, the influence on vertices are indeed `augment' or `weaken' the feature pattern. The third assumption is merely stating the source of overfitting exists. 
\end{remark}
\noindent
We present the following theorem for the graphs with the aforementioned assumptions:
\begin{theorem}
\label{them:dropedge}
Let the interpolation parameter $\gamma$ in \Cref{equ:layergraphconvolution} be $\gamma=\frac{1}{d}$. There exists a way to arrange the \emph{GSA feature relations}, such that for at least $C\cdot (r+1)^{\frac{3}{2}} \cdot \frac{n}{4^{r}}$ ($C>\sqrt{\pi}$ is a constant) vertices, the Global Self-attention mechanism will eliminate the influence of one of its neighboring vertices. 
\end{theorem}
\begin{proof}
This is a natural result following the well-renowned Ramsey theorem. By Ramsey theorem and the assumption, for every $R(r+1,r+1)$ vertices ($R(\cdot, \cdot)$ denotes the Ramsey number), there exist $(r+1)$ vertices with feature influences all positive ($+1$) or all negative ($-1$). Define the set of these $(r+1)$ vertices as $\boldsymbol{V}'$, and we can arrange the \emph{GSA feature relations} in a way that for any $v' \in \boldsymbol{V}'$, vertices $\boldsymbol{V}'/\{v'\}$ will be the $r$ vertices with non-zero \emph{GSA feature relations} to $v'$. \par
Therefore, if we fix a vertex $v' \in \boldsymbol{V}'$, the cumulative influence to $v'$ from the Global Self-attention mechanism is either $+1$ or $-1$ (with the normalization). Pick the interpolation parameter as $\gamma=\frac{1}{d}$, the overall influence from the the Global Self-attention mechanism to $v'$ is either $+\frac{1}{d}$ or $-\frac{1}{d}$. Furthermore, since the feature influences from $\mathcal{N}(v')$ (following the edge geometry) are \emph{not} entirely positive or negative, there must be one vertex in $\mathcal{N}(v')$ with its feature influence eliminated. \par
For all the vertices in $\boldsymbol{V}'$, we have at least $(r+1)$ vertices undergoing the `DropEdge' procedure. Now, suppose in the worst case, the $\boldsymbol{V}'$ sets with mutually positive or negative influences vertices do not overlap; in this case, we can find at least $(r+1)$ qualified vertices among every $R(r+1,r+1)$ vertices. A well-establish upper bound for the $R(s,s)$ Ramsey number (see \cite{bollobas2012graph}) is $[1+o(1)]\frac{4^{s-1}}{\sqrt{\pi s}}$. Plugging the number of $(r+1)$, this will give us $\frac{n}{4^{r}}\cdot \sqrt{\pi}[1+o(1)] \cdot (r+1)^{\frac{1}{2}}$ sets of $\boldsymbol{V}'$ vertices, and an overall number of vertices of at least $\sqrt{\pi}[1+o(1)]\cdot (r+1)^{\frac{3}{2}} \cdot \frac{n}{4^{r}}$ to undergo the `edge elimination' process.
\end{proof}
\begin{remark}
We remark that when assumption 1.1) does not hold, we can modify the Global Self-attention mechanism to a `refined' version that sub-samples $r$ vertices to compute feature influences. In this way, it is possible to design an algorithm that with \emph{any} arrangement of the relations between features, a large portion of vertices will undergo the edge dropout process with high probability. However, as we focus on understanding the GSA mechanism in this paper, we leave the above idea as a future direction to pursue.
\end{remark}
\subsection{Over-smoothing}
\label{subsec:oversmooth}
Our result on over-smoothing is based on the theory in \cite{oono2019graph}, which analyzes the over-smoothing problem as the convergence distance to an invariant subspace. Under the ReLu activation, the distance between the layer output and the space is upper-bounded by $O((s\lambda)^{L})$, where $s$ is the maximum singular value of the weight matrix, $\lambda$ is a parameter related to the Graph Laplacian, and $L$ is the number of layers. We demonstrate that, as long as the features of vertices are \emph{independent}, applying the Global Self-attention mechanism is equivalent to performing convolution with a substituting weight matrix with a larger maximum singular value. We remark that our proof follows an idea similar to \cite{rong2020dropedge}. The difference is that the proof in \cite{rong2020dropedge} established a notion of the $\epsilon$-smoothing layer, while our proof is underpinned by the explicit increment of $s$.\par
In the below sections, unless explicitly specifying, we use shorted-handed notations $\boldsymbol{H} := \boldsymbol{H}^{(l)}$ and $\boldsymbol{W} := \boldsymbol{W}^{(l)}$ to denote the output matrix at any layer. The convolutional result of each layer can therefore be denoted as
\begin{equation}
\label{equ:convolutionoutput}
(\boldsymbol{H} + \gamma\boldsymbol{\tilde{A}}^{-1}\boldsymbol{B}\boldsymbol{H})\boldsymbol{W},
\end{equation}
where $\boldsymbol{\tilde{A}}$ invertible as it is an approximation of the graph Laplacian. We assume the output matrices $\boldsymbol{H}$ satisfy the following property:
\begin{assumption}
$\boldsymbol{H}$ is of full column ranks. Notice that this implies $(\boldsymbol{H}^{T}\boldsymbol{H})$ is invertible, Hermitian and positive definite. We further assume $(\boldsymbol{H}\boldsymbol{H}^{T})$ is invertible.
\end{assumption}
\noindent
Also, we employ the following approximation of $\boldsymbol{B}$ and $\boldsymbol{\tilde{A}}$ in the analysis:
\begin{assumption}
We use $\boldsymbol{\hat{B}}$ as an Hermitian and positive definite approximation of $\boldsymbol{B}$. Moreover, we use $\boldsymbol{\tilde{A}'}=(1+\epsilon)\boldsymbol{I}+\boldsymbol{D}^{-\frac{1}{2}} \boldsymbol{A} \boldsymbol{D}^{-\frac{1}{2}}$ ($\epsilon>0$) as a variation of graph Laplacian to approximate $\boldsymbol{\tilde{A}}$. Notice that $\boldsymbol{\tilde{A}'}$ is also Hermitian and positive definite as it is strictly diagonally dominant.
\end{assumption}
\noindent
\begin{remark}
The above assumptions and approximation are easy to satisfy in real-world applications. Assumption 2 almost merely assumes the features of different dimensions are \emph{independent}. The $\boldsymbol{\hat{B}}$ matrix can be viewed as a pre-Softmax version of $\boldsymbol{B}$, which is indeed Hermitian and the main diagonal of $\boldsymbol{B}$ is at least as large as any other element. Also, $\boldsymbol{\tilde{A}'}$ without the $\epsilon$ term will be Hermitian and positive semi-definite itself -- and the input corresponds to the $0$ eigenvalue is unique.
\end{remark}
\noindent
Denote the output of \Cref{equ:convolutionoutput} as $\boldsymbol{H}\boldsymbol{\tilde{W}}$, with some simple algebraic manipulation, the expression of $\boldsymbol{\tilde{W}}$ will be:
\begin{equation}
\label{equ:wTildeExplicit}
\begin{aligned}
\boldsymbol{\tilde{W}} & = (\boldsymbol{I} + \gamma(\boldsymbol{H}^{T}\boldsymbol{H})^{-1}{\boldsymbol{H}}^{T}\boldsymbol{\tilde{A}}^{-1}\boldsymbol{B}\boldsymbol{H})\boldsymbol{W}\\
& \approx (\boldsymbol{I} + \gamma(\boldsymbol{H}^{T}\boldsymbol{H})^{-1}{\boldsymbol{H}}^{T}\boldsymbol{\tilde{A}'}^{-1}\boldsymbol{\hat{B}}\boldsymbol{H})\boldsymbol{W}.
\end{aligned}
\end{equation}
To simplify the notations, we define $\boldsymbol{\hat{C}}:=\boldsymbol{\tilde{A}'}^{-1}\boldsymbol{\hat{B}}$, which is a Hermitian and positive definite matrix. Finally, we also set $\gamma>0$ in the analysis as it will be otherwise without the GSA mechanism. \par
Now we are ready to introduce the definitions and results in \cite{oono2019graph} as our preliminaries. The main concepts of \cite{oono2019graph} are the invariant subspace and its distance between a matrix. Formally, they are defined as follows:
\begin{definition}[\cite{oono2019graph}]
For $M,N \in \mathbb{N}^{+}$ and $M<N$, suppose there exists a subspace $U \subseteq \mathbb{R}^{M}$ for $\mathbb{R}^{N}$ such that:
\begin{itemize}
\item $U$ has orthonormal basis $\boldsymbol{E}$ consisting of non-negative vectors;
\item $U$ is invariant under $P$;
\end{itemize}
then we define $\mathcal{M}:= U \otimes R^{C} = \{\sum_{i=1}^{M}\boldsymbol{e}_{i} \otimes \boldsymbol{k}_{i} | \boldsymbol{k}_{i} \in \mathbb{R}^{C}\} = \{\boldsymbol{E}\boldsymbol{K}|\boldsymbol{K} \in \mathbb{R}^{M \times C}\}$ as the subspace of $\mathbb{R}^{N \times C}$. Furthermore, for any matrix $\boldsymbol{H} \in \mathbb{R}^{N \times C}$, we define the distance between $\boldsymbol{H}$ and $\mathcal{M}$ as $d_{\mathcal{M}}(\boldsymbol{H}) := \inf_{\boldsymbol{Y}\in\mathcal{M}}\{\|\boldsymbol{H}-\boldsymbol{Y}\|_{F}\}$.
\end{definition}
Following the notions in the literature, we denote the maximum singular value of $\boldsymbol{W}^{l}$ as $s_{l}$, and let $s:=\sup_{l \in \mathbb{N}^{+}} s_{l}$. We recall Theorem 2 and Corollary 2 of \cite{oono2019graph} for a multi-layer GCN to form the following lemma:
\begin{lemma}[\cite{oono2019graph}]
\label{lem:gcnoriginal}
For any initial value $\boldsymbol{H}^{(0)}$ and the output of the $l$-th layer as $\boldsymbol{H}^{(l)}$, the convergence rate to $\mathcal{M}$ is $d_{\mathcal{M}}(\boldsymbol{H}^{(l)})=O((s\lambda)^{l})$, where $\lambda$ is the second-largest eigenvalue of the normalized graph Laplacian. Furthermore, $\boldsymbol{H}^{(l)}$ satisfies $d_{\mathcal{M}}(\boldsymbol{H}^{(l)})\leq(s\lambda)^{l}d_{\mathcal{M}}(\boldsymbol{H}^{(0)})$.
\end{lemma}
In the above lemma we assume $s\lambda<1$, which means $d_{\mathcal{M}}(\boldsymbol{H}^{(l)})$ will exponentially converge to $0$. We are now ready to demonstrate our main theorem for over-smoothing. Let $\mathcal{G}$ be a Graph Convolution Network and $\hat{\mathcal{G}}$ be $\mathcal{G}$ with the Global Self-attention mechianism, we have
\begin{theorem}
\label{thm:mainsmooth}
$\forall$ $l\in\mathbb{N}^{+}$, let $\boldsymbol{H}^{(l)}$ and $\boldsymbol{\hat{H}}^{(l)}$ denote the output of $\mathcal{G}$ and $\hat{\mathcal{G}}$ on the $l$-th layer. As $l\rightarrow \infty$, the convergence rate to $\mathcal{M}$ for $\hat{\mathcal{G}}$ is asymptotically slower than that of $\mathcal{G}$. i.e. $d_{\mathcal{M}}(\boldsymbol{\hat{H}}^{(l)})=\omega(d_{\mathcal{M}}(\boldsymbol{H}^{(l)}))$. Furthermore, there exists $\boldsymbol{H}^{(l)}$ on layer $l$ and the choice of $\gamma$ such that $d_{\mathcal{M}}(\boldsymbol{\hat{H}}^{(l+1)})>d_{\mathcal{M}}(\boldsymbol{H}^{(l+1)})$.
\end{theorem}
Recall from \Cref{equ:wTildeExplicit}, with a slight abuse of notation, let $\boldsymbol{W}$ and $\boldsymbol{\tilde{W}}$ be the weight matrices for $\mathcal{G}$ and $\hat{\mathcal{G}}$ on any layer. The proof of \Cref{thm:mainsmooth} crucially relies on the fact that for each layer, the maximum singular value of $\boldsymbol{\tilde{W}}$ is larger than its counterpart in $\boldsymbol{W}$. En route to proving \Cref{thm:mainsmooth}, we will establish the following results. We use $\lambda_{\min}(\cdot)$ and $\sigma_{\min}(\cdot)$ to denote the minimum eigenvalue and singular value.
\begin{lemma}
\label{lem:hermitiansigularvalue}
The minimum eigenvalue and singular value of matrix $(\boldsymbol{I} + \gamma(\boldsymbol{H}^{T}\boldsymbol{H})^{-1}{\boldsymbol{H}}^{T}\boldsymbol{\hat{C}}\boldsymbol{H})$ in \Cref{equ:wTildeExplicit} are greater than 1. I.e. let $\boldsymbol{P} := (\boldsymbol{I} + \gamma(\boldsymbol{H}^{T}\boldsymbol{H})^{-1}{\boldsymbol{H}}^{T}\boldsymbol{\hat{C}}\boldsymbol{H})$, we have $ \lambda_{\min}(\boldsymbol{P})>1$ and $\sigma_{\min}(\boldsymbol{P})>1$.
\end{lemma}
\begin{proof}
We first prove the matrix $\boldsymbol{Q} :=\gamma(\boldsymbol{H}^{T}\boldsymbol{H})^{-1}{\boldsymbol{H}}^{T}\boldsymbol{\hat{C}}\boldsymbol{H}$ is Hermitian and  positive definite. According to Assumption 2, since $\boldsymbol{H}^{T}\boldsymbol{H}$ is Hermitian and positive definite, $(\boldsymbol{H}^{T}\boldsymbol{H})^{-1}$ is Hermitian and positive definite; and since $\boldsymbol{\hat{C}}$ is Hermitian and positive definite, ${\boldsymbol{H}}^{T}\boldsymbol{\hat{C}}\boldsymbol{H}$ is Hermitian and positive definite. Thus, $\boldsymbol{Q}$ is positive definite by multiplication. We can show $\boldsymbol{Q}$ is Hermitian by showing $\boldsymbol{\bar{Q}}$ is Hermitian:
\begin{align*}
\boldsymbol{\bar{Q}} &= (\boldsymbol{H}^{T}\boldsymbol{H})^{-1}{\boldsymbol{H}}^{T}\boldsymbol{\hat{C}}\boldsymbol{H}\\
&= (\boldsymbol{H}^{T}\boldsymbol{H})^{-1}{\boldsymbol{H}}^{T}(\boldsymbol{H}\boldsymbol{H}^{T})(\boldsymbol{H}\boldsymbol{H}^{T})^{-1}\boldsymbol{\hat{C}}\boldsymbol{H}\\
&=\boldsymbol{H}^{T}(\boldsymbol{H}\boldsymbol{H}^{T})^{-1}\boldsymbol{\hat{C}}\boldsymbol{H}\\
&=\boldsymbol{H}^{T}\boldsymbol{\hat{C}}\boldsymbol{\hat{C}}^{-1}(\boldsymbol{H}\boldsymbol{H}^{T})^{-1}\boldsymbol{\hat{C}}\boldsymbol{H}\\
&=\boldsymbol{H}^{T}\boldsymbol{\hat{C}}(\boldsymbol{H}\boldsymbol{H}^{T}\boldsymbol{\hat{C}})^{-1}\boldsymbol{\hat{C}}\boldsymbol{H}\\
&=\boldsymbol{H}^{T}\boldsymbol{\hat{C}}(\boldsymbol{H}\boldsymbol{H}^{T}\boldsymbol{\hat{C}})^{-1}\boldsymbol{\hat{C}}\boldsymbol{H}(\boldsymbol{H}^{T}\boldsymbol{H})(\boldsymbol{H}^{T}\boldsymbol{H})^{-1}\\
&=\boldsymbol{H}^{T}\boldsymbol{\hat{C}}(\boldsymbol{H}\boldsymbol{H}^{T}\boldsymbol{\hat{C}})^{-1}(\boldsymbol{H}\boldsymbol{H}^{T}\boldsymbol{\hat{C}})\boldsymbol{H}(\boldsymbol{H}^{T}\boldsymbol{H})^{-1}\\
&=\boldsymbol{H}^{T}\boldsymbol{\hat{C}}\boldsymbol{H}(\boldsymbol{H}^{T}\boldsymbol{H})^{-1} = \boldsymbol{\bar{Q}}^{T}.
\end{align*}
Since $\gamma>0$, we can conclude that $\boldsymbol{Q}$ is Hermitian, and its minimum eigenvalue $\lambda_{\min}(\boldsymbol{Q})>0$. Now we introduce a theorem in \cite{Knutson2001HoneycombsAS}, which says that if $\boldsymbol{A}$ and $\boldsymbol{B}$ are Hermitian matrices with eigenvalues denoted in descending order. i.e.  $\lambda_{1}(\boldsymbol{A})>\lambda_{2}(\boldsymbol{A})>\cdots>\lambda_{n}(\boldsymbol{A})$ and $\lambda_{1}(\boldsymbol{B})>\lambda_{2}(\boldsymbol{B})>\cdots>\lambda_{n}(\boldsymbol{B})$, and if $\boldsymbol{A}+\boldsymbol{B}=\boldsymbol{C}$ with the eigenvalue of $\boldsymbol{C}$ denoted in descending order. i.e.  $\lambda_{1}(\boldsymbol{C})>\lambda_{2}(\boldsymbol{C})>\cdots>\lambda_{n}(\boldsymbol{C})$, then there is inequality:
\begin{equation*}
\lambda_{n-i-j}(\boldsymbol{C}) \geq \lambda_{n-i}(\boldsymbol{A})+\lambda_{n-j}(\boldsymbol{B}).
\end{equation*}
If we set $i=0,j=0$ for the above inequality, and consider the condition that identity matrix is Hermitian and with eigenvalue $1$, then we will have:
\begin{align*}
\lambda_{\min}(\boldsymbol{P}) \geq \lambda_{\min}(\boldsymbol{I})+\lambda_{\min}(\boldsymbol{Q})>1 .
\end{align*}
And since $\boldsymbol{P}$ is obviously Hermitian, which means its smallest singular value should be the
absolute value of its smallest eigenvalue. Therefore, we have $\sigma_{\min}(\boldsymbol{P})>1$.
\end{proof}
Now we give the following proposition to help compare the singular values:
\begin{prop}
\label{prop:2normanalysis}
Let $\|\cdot\|$ denotes the 2-norm, and given two matrices $\boldsymbol{A}$ and $\boldsymbol{B}$ where $\boldsymbol{A}$ is a square matrix, then the inequalities $\sigma_{\min}(\boldsymbol{A})\|\boldsymbol{B}\|\leq\|\boldsymbol{A}\boldsymbol{B}\|\leq\sigma_{\max}(\boldsymbol{A})\|\boldsymbol{B}\|$ hold.
\end{prop}
The proof of proposition \Cref{prop:2normanalysis} can be found in most matrix analysis textbooks. Now we can demonstrate the relation between the singular values of $\boldsymbol{\tilde{W}}$ and $\boldsymbol{W}$.
\begin{lemma}
\label{lem:singularvaluecomparison}
$\forall$ $l\in\mathbb{N}^{+}$, let $s^{(l)}$ and $\tilde{s}^{(l)}$ be the maximum singular values of $\boldsymbol{W}^{(l)}$ and $\boldsymbol{\tilde{W}}^{(l)}$, respectively. Then there is $\tilde{s}^{(l)}>s^{(l)}$.
\end{lemma}
\begin{proof}
The proof follows a simple combination of the results of lemma \Cref{lem:hermitiansigularvalue} and proposition \Cref{prop:2normanalysis}. Let $||\cdot||$ be a 2-norm, then one will have:\\
\begin{align*}
\tilde{s}^{(l)} &= \|\boldsymbol{\tilde{W}}^{(l)}\|\\
&=\|(\boldsymbol{I} + \gamma(\boldsymbol{H}^{T}\boldsymbol{H})^{-1}{\boldsymbol{H}}^{T}\boldsymbol{\hat{C}}\boldsymbol{H})\boldsymbol{W}^{(l)}\|\\
& \geq \sigma_{\min}(\boldsymbol{I} + \gamma(\boldsymbol{H}^{T}\boldsymbol{H})^{-1}{\boldsymbol{H}}^{T}\boldsymbol{\hat{C}}\boldsymbol{H}) \|\boldsymbol{W}^{(l)}\| \tag{proposition \Cref{prop:2normanalysis}}\\
&> \|\boldsymbol{W}^{(l)}\| = s^{(l)}\tag{lemma \Cref{lem:hermitiansigularvalue}}
\end{align*}
\end{proof}
\noindent
\textbf{Proof of \Cref{thm:mainsmooth}}\\
The asymptotic relation on convergence rate can be shown straightforwardly by taking the limit of the rate in lemma \Cref{lem:gcnoriginal}. Let $s=\sup_{l \in \mathbb{N}^{+}} s_{l}$ and $\hat{s}=\sup_{l \in \mathbb{N}^{+}} \hat{s}_{l}$ be the largest singular values among all layers, one can get:
\begin{align*}
\lim_{l \to \infty} \frac{d_{\mathcal{M}}(\boldsymbol{\hat{H}}^{(l)})}{d_{\mathcal{M}}(\boldsymbol{H}^{(l)})} = \lim_{l \to \infty} (\frac{\hat{s}}{s})^{l} = \infty \tag{$\frac{\hat{s}}{s}>1$}
\end{align*}
The proof of the second part of the theorem relies on the Proposition 4 of \cite{oono2019graph}, which stated that for some $\boldsymbol{H}$, if $s^{(l)}\lambda>1$, we will have $d_{\mathcal{M}}(\boldsymbol{H}^{(l+1)})>d_{\mathcal{M}}(\boldsymbol{H}^{(l)})$. Thus, for layer $l$, if we choose $\gamma$ so that $s^{(l)}\lambda<1$ and $\hat{s}^{(l)}\lambda>1$, then there exists input $\boldsymbol{H}^{(l)}=\boldsymbol{\hat{H}}^{(l)}$ such that:
\begin{align*}
d_{\mathcal{M}}(\boldsymbol{H}^{(l+1)}) &\leq (s^{(l)}\lambda)d_{\mathcal{M}}(\boldsymbol{H}^{(l)})\\
&<(\hat{s}^{(l)}\lambda)d_{\mathcal{M}}(\boldsymbol{\hat{H}}^{(l+1)})\\
&<d_{\mathcal{M}}(\boldsymbol{\hat{H}}^{(l+1)})
\end{align*}
\section{Empirical Analysis for the GSA Mechanism on GCNs}
\label{sec:Experiments}
In this section, we empirically evaluate the impact of the GSA mechanism on GCNs in terms of the expressive power and the abilities to mitigate overfitting and over-smoothing. We organize the layers in the same way of \Cref{equ:layergraphconvolution}, and simply put it over the plain GCNs. For a convinient notation, we name the above model as \emph{GSA-GCN} in this section. We evaluate its performance with multiple benchmark datasets on two tasks: (I) node classification with citation network datasets: Cora, Citeseer and Pubmed \cite{sen2008collective}, aiming to classify academic papers into various subjects, (II) graph classification on the COIL-RAG dataset. We split the experiments into three parts to show the respective properties of the GSA-GCN. For the node classification task, we firstly apply the GSA-GCN to both supervised and semi-supervised node classifications. The results demonstrate enhanced accuracy and overfitting resistance for the GSA-GCN. Secondly, to corroborate the theorem for the GSA mechanism to mitigate over-smoothing, we compare the training accuracy for GCNs and GSA-GCNs as the number of layers goes deep. And finally, we extend the comparison between GSA-GCNs and plain GCNs to graph classification task with the COIL-RAG dataset to illustrate the preferable expressive power of GSA-GCNs beyond classical node classification. The metadata of aforementioned datasets are illustrated in table \Cref{tab:datasets}.\par

\begin{table}[!h]
\caption{Datasets used in the Experiments} \label{tab:datasets}
\centering
\resizebox{0.6\textwidth}{!}{
\begin{tabularx}{10cm}{ c c c c c c }
 \toprule[1.2pt]
\textbf{Dataset} & \textbf{Graphs} & \textbf{Nodes} & \textbf{Edges} & \textbf{Classes} & \textbf{Features} \\\midrule[0.8pt]
Cora & 1 & 2708 & 5429 & 7 & 1433 \\ 
Citeseer& 1 & 3327 & 4732 & 6 & 3703 \\ 
Pubmed & 1 & 19717 & 44338 & 3 & 500 \\ 
COIL-RAG & 3900 & 3.01 & 3.02& 100 (graph) & 64\\
\bottomrule[1.2pt]
\end{tabularx}}
\end{table}
\subsection{Semi- and Full-supervised Node Classification}
The semi-supervised experiment follows the setup in \cite{kipf2016semi}, and the full-supervised experiment follows the practice in \cite{rong2020dropedge,chen2018fastgcn}. We compare the performance of GSA-GCN with several recently-proposed compelling models based on GCNs, and the two-layer GCN backbone is considered as baseline unless otherwise specified. For semi-supervised learning, only a marginal portion of labeled instances are used in training process, thus it yields weaker performance than full-supervised learning. \Cref{tab:semi_node_class} demonstrates test accuracy on three datasets by GCNs \cite{kipf2016semi}, Graph Attention Networks \cite{velivckovic2017graph}, Jumping Knowledge Networks \cite{xu2018representation} and DropEdge-GCN \cite{rong2020dropedge}. We also summarize previous state-of-the-art testing accuracy on three datasets with full-supervised learning in \Cref{tab:full_node_class}. The performances of comparison models are either as reported in original paper or obtained by fine-tuning to the best performance as we are able to achieve.

\begin{table}[!h]
\begin{center}
{\caption{Semi-supervised Node Classification Accuracy(\%)}\label{tab:semi_node_class}}
\begin{tabular}{c c c c}
 Model & Cora & Citeseer & Pubmed \\\midrule[0.8pt]
 GCN & 81.5 & 70.3 & 79.0\\ \hline
 GAT & 83.0 & 72.5 & 79.0 \\ \hline
 JK-Net (4) & 80.2 & 68.7 & 78.0 \\  \hline
 DropEdge-GCN & 82.8 & 72.3 & 79.6\\ \hline
 GSA-GCN & \textbf{83.3} & \textbf{72.9}& \textbf{80.1}\\ \bottomrule[1.2pt]
\end{tabular}
\end{center}
\end{table}

\begin{table}[!h]
\begin{center}
{\caption{Full-supervised Node Classification Accuracy(\%)}\label{tab:full_node_class}}
\begin{tabular}{c c c c}
 Model & Cora & Citeseer & Pubmed \\\midrule[0.8pt]
 GCN & 86.1 & 75.9 & 90.2\\  \hline
 GAT & 86.4 & 76.6 & OOM \\ \hline
 JK-Net & 86.9 & 78.3 & 90.5 \\ \hline
 DropEdge-GCN & 86.5 & 78.7 & \textbf{91.2}\\ \hline 
 GSA-GCN & \textbf{88.2} & \textbf{79.1}& 89.4 \\ \bottomrule[1.2pt]
\end{tabular}
\end{center}
\end{table}

As illustrated in \Cref{tab:semi_node_class} and \Cref{tab:full_node_class}, the GSA-GCN model outperforms other advanced methods on almost all the datasets for both tasks. The only exception is the full-supervised classification on Pubmed dataset. Notice some other state-of-the-art methods apart from the ones listed in the tables \emph{may} have reported a higher accuracy than GSA-GCN. Nevertheless, our intention is \emph{not} to show the superiority of the GSA-GCN model, but to verify the \emph{merits of the GSA mechanism}. To this end, the results are sufficiently convincing as the simple GSA mechanism on plain GCNs can lead to performances surpassing the listed advanced methods, which are carefully-designed and complicated. In summary, consistent with the intuition, the GSA mechanism indeed introduce compelling expressive power to the GCNs. Furthermore, the significant performance improvement for the semi-supervised classification task ratifies the capacities for the GSA mechanism to mitigate overfitting. 

\subsection{Over-smoothing}
To provide empirical evidences for \Cref{thm:mainsmooth}, we show the changes of training accuracy with an increasing number of layers on plain GCN and GSA-GCN. Two fully-supervised node classification datasets (Cora and Citeseer) are employed in this part of the experiment. Notice that since over-smoothing will obstruct the accuracy optimization on both training and testing sets, demonstrating the results on the \emph{training} accuracy alone is sufficient for our purpose.\par

The results can be summarized in \Cref{fig:cora_oversmoothing} and \Cref{fig:citeseer_oversmoothing}. From the figures, it can be observed that GSA-GCN can consistently achieve a higher training accuracy than its plain GCN counterpart. Moreover, as the graph network deepens to a further extent, the gap becomes more significant. We observe from the training dynamic that the training accuracy of plain GCN for deep models typically stuck at a low level after few epochs, while the curves in GSA-GCN illustrate a fluctuating pattern although it does not converge to a high value.

\begin{figure}[!h]
\centering
\includegraphics[width=0.45\textwidth]{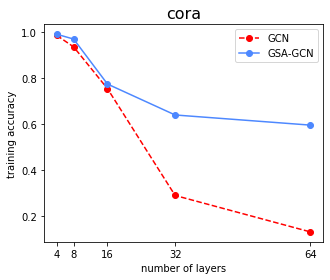}
\caption{\label{fig:cora_oversmoothing} Over-smoothing with model depth on Cora}
\end{figure}
\par

\begin{figure}[!h]
\centering
\includegraphics[width=0.45\textwidth]{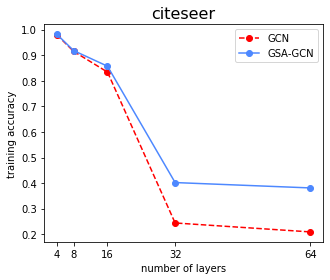}
\caption{\label{fig:citeseer_oversmoothing} Over-smoothing with model depth on Citeseer}
\end{figure}
\par

\subsection{Graph Classification}
Limited by space, for the graph classification task, we only illustrate the results on the COIL-RAG dataset. We tested the results on plain and GSA-GCN, and referred an external result from \cite{atamnaSPIGCN}. Our read-out layer for graph classification is based on summation pooling, and the weight regularization is set to $1e-4$. The results are summarized in \Cref{tab:graphclassification}. 
\begin{table}[!h]
\begin{center}
{\caption{Graph Classification Performance of GCNs and GSA-GCNs}\label{tab:graphclassification}}
\begin{tabular}{p{3cm} c c}
Model & Training Accuracy & Test Accuracy  \\\midrule[0.8pt]
GCN & 95.35 & 85.15  \\ \hline
SPI-GCN \cite{atamnaSPIGCN} & N.A. & 75.72  \\ \hline 
GSA-GCN & \textbf{97.34} & \textbf{88.28} \\ \bottomrule[1.2pt]
\end{tabular}
\end{center}
\end{table}
\par
From the table, it can be observed that the GSA-GCN model can provide the optimal performance among the methods implemented. This positive result indicates that the gain on the expressive power from the GSA mechanism is not limited to node classification tasks. To further illustrate the insights of the GSA-GCN, we plot the training and testing accuracy curves with respect to the number of epochs in \Cref{fig:coilrag}. From the figure, it can be observed that the curves of both the training and testing accuracy of GSA-GCN stand to a higher level than their plain GCN counterpart. The curve also supports the overfitting-resistance property of the GSA mechanism: we can observe that in the later epochs, the testing accuracy of the plain GCN starts to decline, while the test accuracy for the GCN-GSA continues to climb.

\begin{figure}[!h]
\centering
\includegraphics[width=0.5\textwidth]{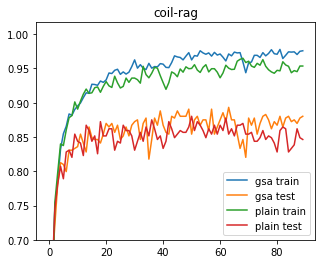}
\caption{\label{fig:coilrag} Graph Classification Accuracy for Plain and GSA GCNs}
\end{figure}
\par

\section{Conclusion}
\label{sec:Conclusion}
In this paper, we study the impact of the Global Self-attention mechanism on Graph Convolutional Networks. To the best of the our knowledge, this is the first concrete attempt to understand such mechanism on GCNs. We first provide a straightforward way to implement the GSA mechanism to GCNS. Subsequently, we theoretically prove that GSA mechanism can mitigate overfitting and over-smoothing problems in GCN-based models based on some recent results. Experiments on two classical tasks illustrate superior expressive power of GSA mechanism against advanced (and much more complicated) variations of the GCN. Furthermore, the theoretical results on the alleviation of overfitting and over-smoothing are reflected in the empirical results.

\bibliographystyle{unsrt}  
\bibliography{references}  
\newpage

\end{document}